\RequirePackage[2020-02-02]{latexrelease}
\documentclass[10pt,twocolumn]{ieeeconf}
\IEEEoverridecommandlockouts                              
\usepackage[utf8]{inputenc}
\usepackage[T1]{fontenc}
\usepackage[printwatermark]{xwatermark}

\usepackage[algo2e]{algorithm2e} 
\usepackage{algorithm}
\usepackage{amsmath}
\usepackage{amsfonts}
\usepackage{amssymb}
\usepackage{mathtools}
\usepackage{xcolor}
\usepackage{amsthm}
\SetKwComment{Comment}{/* }{ */}
\usepackage[font=footnotesize,skip=0pt]{caption}

\newtheorem{theorem}{Theorem}
\newtheorem{lem}[theorem]{Lemma}
\newtheorem{definition}{Definition}

\newtheorem{assumption}{Assumption}
\newtheorem{problem}{Problem}

\usepackage{graphicx}
\usepackage{subfig}
\newcommand{\mathbold}[1]{\boldsymbol{#1}}
\usepackage{flushend}
\usepackage{float}
\usepackage{lipsum}  

\title{\LARGE \bf Moving Obstacle Avoidance: a Data-Driven Risk-Aware Approach}

\author{Skylar X. Wei$^{*}$, Anushri Dixit$^{*}$, Shashank Tomar, and Joel W. Burdick
\thanks{$^{*}$Both authors contributed equally.}
\thanks{The authors are with the Division of Engineering $\&$ Applied Science, California Institute of Technology,  MC 104-44,
Pasadena, CA 91125, (\{swei, adixit, stomar, jburdick\}@caltech.edu). }}%

\begin{document}
\maketitle
\thispagestyle{empty}
\pagestyle{empty}

\begin{abstract}
This paper proposes a new structured method for a moving agent to predict the paths of dynamically moving obstacles and avoid them using a risk-aware model predictive control (MPC) scheme. Given noisy measurements of the a priori unknown obstacle trajectory, a bootstrapping technique predicts a set of obstacle trajectories. The bootstrapped predictions are incorporated in the MPC optimization using a risk-aware methodology so as to provide probabilistic guarantees on obstacle avoidance. We validate our methods using simulations of a 3-dimensional multi-rotor drone that avoids various moving obstacles, such as a thrown ball and a frisbee with air drag. 
\end{abstract}


\section{INTRODUCTION}
\vspace{-1mm}
Emerging applications of robots in urban, cluttered, and potentially hostile environments have increased the importance of online path planning with obstacle behavior classification, and avoidance~\cite{fan2021step}. Traditionally, the interaction of a robot with an obstacle is formulated as the problem of planning a collision-free path to navigate from a starting position to a goal \cite{latombe_1996}. In dynamic environments with an arbitrary number of moving obstacles, and agents with bounded velocity, this problem is known to be NP-hard ~\cite{canny1988complexity}.
\vspace{-1mm}

One way to handle dynamic obstacles is to limit their modeled motions. In~\cite{classification_highcite}, the authors assumed a priori knowledge of the obstacle dynamics or motion patterns.  
Or, one can plan the agent's path off-line using a Probabilistic Roadmap (PRM) in a field of static obstacles, and then replan when dynamical behaviors are observed~\cite{PRM_cite}. However, without prior knowledge of the obstacle behavior, a worst-case analyses of unsafe obstacle locations can lead to overly conservative behaviors.  Potential fields (PFs) have been actively used for dynamic obstacle avoidance: e.g., recent works   \cite{APF_cite1} apply artificial PFs with stochastic reachable sets in Human-Centered environments. Despite its computational efficiency and scalability with the number of obstacles, slow moving and simple (linear or double integrator-like) dynamics are assumed. Switching-based planning methods detect and classify dynamic obstacle behavior against a set of trajectories, such as constant speed, linear, and projectile-like motion~\cite{ switch_based_trajectory_classifier2, switch_based_trajectory_classifier3}.  Classification-based methods require distinguishable obstacle behaviors and prior knowledge about the dynamic environment to a generate set of trajectories.

This paper presents a new framework for discovering the dynamics of a priori unknown moving obstacles, forecasting their trajectories, and providing risk-aware optimal avoidance strategies. It replaces the need for obstacle trajectory/model classification, while allowing an online implementation. Extracting a dynamics model from data is challenging \cite{brunton2017chaos} and the difficulty increases when the available data is limited, noisy, and partial. To tackle the partial measurement issue, we leverage Takens embedding theorem ~\cite{takens1981detecting}, which enables partial observations to produce an attractor that is diffeomorphic to the full-state attractor.  
We then use Singular Spectrum Analysis (SSA)~\cite{SSAbook,ssa_withR} to separate noise from the underlying signal and to extract a recurrence model to predict the obstacle behavior. Note that our use of time delay embedding is also the basis of Eigensystem Realization Algorithm (ERA) in linear system identification \cite{ERA} and has been connected with Koopman operators \cite{brunton2017chaos}. Inspired by \cite{ensemble_sindy}, we employ a classical bootstrap method to forecast a set of moving obstacle trajectories with statistical quantification. 
We propose an MPC planner that incorporates the set of obstacle forecasts  as an affine conservative approximation of a distributionally-robust chance constraint.  This constraint is then efficiently reformulated in a risk-aware way, allowing the MPC optimization to solved using a sequential convex programming approach~\cite{SCP1,scp2}.


We demonstrate our approach on three scenarios that exhibit increasingly complicated dynamical behavior.  Monte-Carlo simulations verify the planner's ability to uphold the user chosen chance constraint. The risk-aware reformulation not only gives provable probabilistic collision avoidance guarantees, but also allows on-line execution of the planner.


\textbf{Notation: } The set of positive integers, natural number, real numbers, and positive real numbers are denoted as $\mathbb{Z}_+$, $\mathbb{N}$, $\mathbb{R}$, and $\mathbb{R}_{+}$, respectively. We denote the sequence of consecutive integers $\{i,i+1,\cdots,i+k\}$ as $\mathbb{Z}^{i:i+k}$. The finite sequence $\{a_1,\cdots,a_k\}$ of a scalar or vector variable $a$ is denoted as $\{a\}_{1}^{k}$. The expression $I_{n\times n}$ is used to denote $n$ by $n$ identity matrices and $\mathbold{1} = [1,1,1]^T$.

\section{SSA Preliminaries} \label{sec:prelims}
\vspace{-1mm}
\begin{figure*}[!ht]
\vspace{2mm}
    \centering
\includegraphics[width=17.5cm]{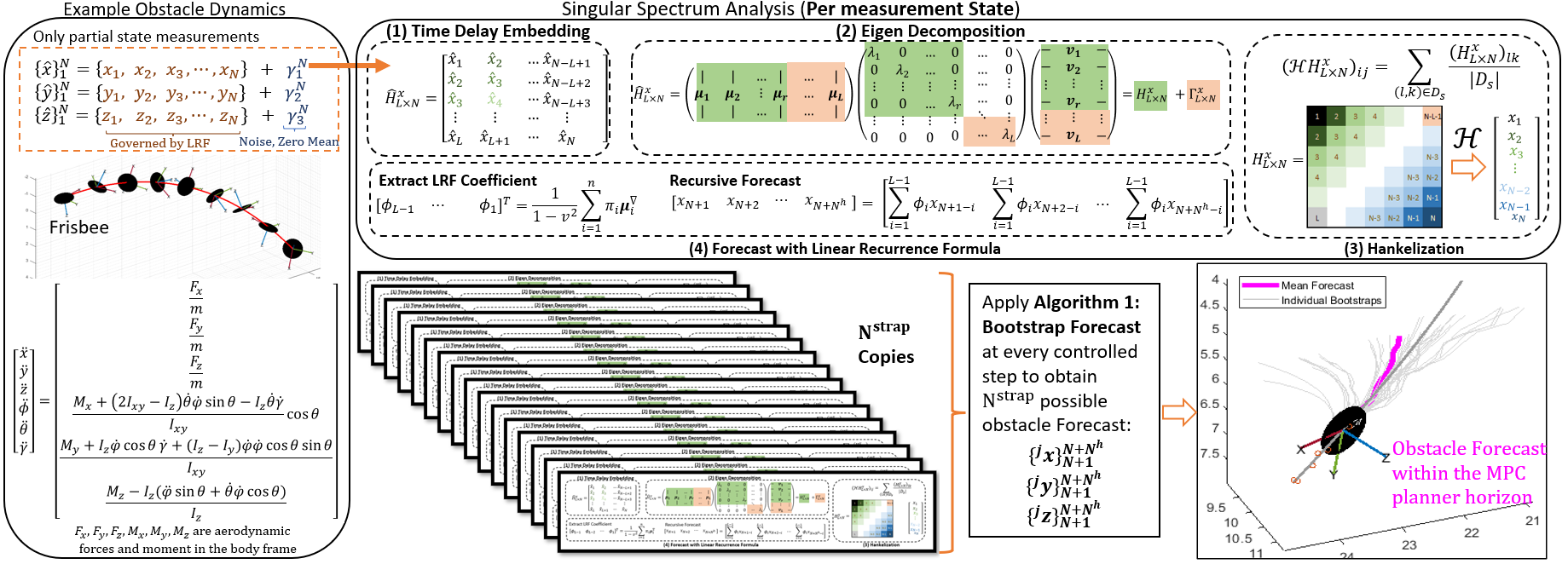}
        \caption{A description of bootstrap-SSA-forecast architecture in forecasting the trajectory of a Frisbee where the stochastic observables (corrupted by zero-mean, noise) consist of $\{\hat{\mathbold{o}}\}_{1}^{N} = [\{\hat{x}\}_1^N,\{\hat{y}\}_1^N,\{\hat{z}\}_1^N]$, the Frisbee's center positions with respect to an inertial frame. The SSA analysis and bootstrap forecast is applied to every observable states as indicated. Despite its 12-state governing dynamics~\cite{hummel2003frisbee} and with only center position measurements of the Frisbee, we show an example $N^{\text{strap}}$ forecasts of the Frisbee trajectory for future time steps $\{1,2,\cdots,N^h\}$ using our proposed framework.}
    \label{fig:MSSA_architecture}
    \vspace{-5mm}
  \end{figure*}
Consider a discrete-time multivariate stochastic process $\{o^{m}\}_{1}^{N}$ where $m$ denotes the $m^{\text{th}}$ observable measurement of the process, and $N$ is the total number of available observations, i.e., $o^m_i$ denotes the $m^{th}$ observation variable at process sampling index $i$. Suppose the true model of the stochastic process in terms of the observables is
\vspace{-1mm}
\begin{equation} \label{eq:signal_struc}
    \underbrace{\begin{bmatrix}\hat{o}^{m}_{1} & \cdots & \hat{o}^{m}_{N}  \end{bmatrix}}_{ \hat{\mathbold{o}}^{m}} \!=\! \underbrace{\begin{bmatrix}o^{m}_{1} & \cdots& o^{m}_{N}  \end{bmatrix}}_{\mathbold{o}^m} \!+\!\underbrace{\begin{bmatrix}\gamma^{m}_{1} &  \cdots & \gamma^{m}_{N}  \end{bmatrix}}_{\mathbold{\gamma}^{m}}
    \vspace{-1mm}
\end{equation}
where~\small$\mathbold{\gamma}^{m}$\normalsize~denotes a random discrete-time zero-mean measurement noise, and~\small$\mathbold{o}^m$\normalsize~is the noiseless observable that captures the  governing laws, which can be composed of \textit{trends}, \textit{seasons}, and \textit{stationary} time series.  Singular Spectrum Analysis \cite{SSAbook} separates the true signal $\mathbold{o}^m$ and the  noise $\mathbold{\gamma}^m$ and extracts a recursive governing dynamic model of $\mathbold{o}^{m}$ that can generate a short term accurate forecast. 


\subsubsection{Time Delay Embedding}
Taken's  method of delays \cite{takens1981detecting} can be used to reconstruct qualitative features of the full-state, phase space from delayed partial observations.  The $m^{\text{th}}$ state raw observables $\hat{\mathbold{o}}^m$ are delay embedded into the following trajectory (Hankel) matrix:
\vspace{-1mm}
\begin{align} \label{eq:hankel}
    H^{m}_{[L,N]}\! &=\!  \left[\begin{smallmatrix}
        \hat{o}^m_{1} & \hat{o}^m_{2}& \dots &  \hat{o}^m_{N-L+1} \\
        \hat{o}^m_{2} & \hat{o}^m_{3}& \dots &  \hat{o}^m_{N-L+2} \\
        \vdots & \vdots  & \ddots &\vdots\\
        \hat{o}^m_{L} & \hat{o}^m_{L+1}& \dots &  \hat{o}^m_{N} \\
    \end{smallmatrix}\right]
    \vspace{-2mm}
\end{align}
where $L$ represents time delay length and $N$ is length of the time series. The repeating patterns in the Hankel matrix represent underlying trends and oscillations, and can be extracted from its covariance matrix: \small$X^{m} = H^{m}_{[L,N]}(H^{m}_{[L,N]})^{T}$\normalsize. 


\subsubsection{Eigen Decomposition}
To recover the true signal $\mathbold{o}^m$, we aim to find the best, low-rank matrix approximation of the true signals by thresholding the eigenvalues of $X^{m}$, similar to~\cite{abdi2010principal}. More specifically, the symmetric covariance matrix $X^{m}$ has a spectral decomposition  $U\Sigma U^T$, where $\Sigma$ is a diagonal matrix with real eigenvalues $\lambda_1 \geq \lambda_2 \geq \cdots \lambda_L$. The matrix of left eigenvectors $U = \begin{bmatrix}\mathbold{\mu}_1, \cdots, \mathbold{\mu}_L \end{bmatrix}$ is orthogonal. The truncated right eigenvectors $V = [\mathbold{\nu}_1,\cdots,\mathbold{\nu}_L]^{T} \in \mathbb{R}^{L\times N}$ of $X^m$ can be found as $V = U\Sigma$. Suppose $\mathbold{\lambda}$ is the optimal threshold 
and $\lambda_n \geq \mathbold{\lambda} \geq \lambda_{n+1}$, which partitions the Hankel matrix~\small $H^{m}_{[L,T]}$\normalsize~as:
\vspace{-3mm}
\begin{equation} \label{eq:noise_signal_separation}
    H^m_{[L,T]}  = \underbrace{\sum_{p = 1}^{n} \sqrt{\lambda_p}\mathbold{\mu}_p\mathbold{\nu}_p^T}_{\triangleq H_{[L,K]}^o } + \underbrace{\sum_{p=n+1}^{L}\sqrt{\lambda_p}\mathbold{\mu}_p\mathbold{\nu}_p^T}_{\triangleq H_{[L,K]}^\gamma}.
    \vspace{-2mm}
\end{equation}


\subsubsection{Hankelization}
Matrix~\small$H_{[L,K]}^{o}$\normalsize \eqref{eq:noise_signal_separation} should maintain a Hankel structure, and minor variations in its $k^{th}$ secondary diagonals result from insufficient noise removal.\footnote{The $k^{\text{th}}$ secondary diagonals of a matrix $M$ are also the $k^{\text{th}}$ diagonals of  $M$ flipped horizontally with respect to its middle column.}~Therefore, a Hankelization step is introduced to perform secondary diagonal averaging, finding the matrix~\small$\mathcal{H}O$\normalsize~that is closest to~\small$H_{[L,K]}^o$\normalsize~with respect to the Frobenius norm among all Hankel matrices of size \small$L\times N$\normalsize~\cite{SSAbook}.  The operator~\small$\mathcal{H}$\normalsize~acting on an arbitrary~\small$L\times N$\normalsize~ matrix~\small$H_{[L,N]}^{y}$\normalsize~entry wise is defined as follows: for the~\small$(i,j)^{th}$\normalsize~element of matrix~\small$H_{[L,N]}^{o}$\normalsize~and~\small$i+j = s$\normalsize, define a set~\small$D_s \triangleq \{(l,n): l+n = s, l\in\mathbb{Z}^{1:L}, n\in \mathbb{Z}^{1:N}\}$\normalsize, is mapped to $(i,j)^{\text{th}}$ element of the hankelized~\small$\mathcal{H}H_{[L,N]}^{o}$\normalsize~via the expression in Fig.\ref{fig:MSSA_architecture} (for the case of~\small$\mathbold{o}^m = \mathbold{x}$\normalsize),
where~\small$|D_s|$\normalsize~denotes the number of elements in the set~\small$D_s$
\normalsize. 

\subsubsection{Forecast with Linear Recurrence Formula}

\begin{definition} \label{defn:Lorder}
A time series $Y_N=\{y\}_{1}^{N}$ admits an \textit{L-decomposition of order not larger than d}, denoted by $\mbox{ord}_L(Y_N)\leq d$, if there exist two systems of functions %
$
    \varrho_k:\mathbb{Z}^{0:L-1}\to \mathbb{R},  \vartheta_k:\mathbb{Z}^{0:L-1}\to \mathbb{R},
$
such that 
$
    y_{i+j} = \sum_{k=1}^{d} \vartheta(i)\varrho_k(j) \quad {i,j}\in {\mathbb{Z}^{0:L-1}\times \mathbb{Z}^{0:L-1}}
$ for all $k\in \mathbb{Z}^{1:d}$.
\end{definition}
\vspace{-1mm}
If $\mbox{ord}_L(Y_N) =  d$, then the series $Y_N$ admits a \textit{L-decomposition of the order d} and both systems of functions $(\varrho_1,\cdots,\varrho_d)$ and $(\vartheta_1,\cdots,\vartheta_d)$ are linearly independent \cite{Golyandina2001AnalysisOT}.
\begin{definition} \label{defn:LRF}
A time series $\{y\}_{1}^{N}$ is governed by a linear recurrent relations/formula (LRF), if there exist coefficients $\{\phi\}_{1}^{m}$ and $\phi_m \neq 0$ such that 
\vspace{-1mm}
\begin{equation} \label{eq:LRF}
    y_{i+d} = \sum_{k=1}^{d}\phi_{k}y_{i+d-k}, \quad \forall
     i \in \mathbb{Z}^{0:N-d}, d<N\ .
\end{equation}
\end{definition}
\vspace{-2mm}
Real-valued time series governed by LRFs consists of sums of products of polynomials (trends), exponentials (stationary, linear time invariant) and sinusoids (seasons)~\cite{SSAbook}.


\begin{theorem}\cite{SSAbook} \label{thm:LRF}
Let $\mathbold{\mu}_i^{1:L-1}$ be the vector of the first $L-1$ components of a left eigenvector $\mathbold{\mu}_i$ of $H_{[L,N]}^m$, and let $\pi_i$ be the $L^{\text{th}}$ component of eigenvector $\mathbold{\mu}_i$. Let $v^2 \triangleq \sum_{i= 1}^{d} \pi_i^2$. Under Assumptions \ref{assum:assumption2} and \ref{assum:assumption3} (see below), the LRF coefficients $\phi_i$ where $ i\in [1,L-1]$  can be computed as:
\vspace{-1.5mm}
\begin{equation} \label{eq:recurrence_coeff}
    \begin{bmatrix}
    \phi_{L-1}&
    \phi_{L-2} &
    \cdots &
    \phi_1
    \end{bmatrix}^{T} = \frac{1}{1-v^2}\sum_{i=1}^{d} \pi_i \mathbold{\mu}_i^{1:L-1}
    \vspace{-1mm}
\end{equation}
and $\mathbold{y}$ evolves as the LRF:
$ \mathbold{y}_{N+1 } = \sum_{j=1}^{L-1}\phi_j\mathbold{y}_{N-j}$.
\end{theorem}

 
\vspace{-2mm}
\section{Problem Statement} \label{sec:modeling}
\vspace{-1mm}

Consider the linear discrete-time dynamical agent model:
\vspace{-1mm}
\begin{align} \label{eq:DT_system_dynamics}
    \mathbold{x}_{i+1} =  A\mathbold{x}_{i} + B\mathbold{u}_{i},  \quad \quad   \mathbold{y}_{i+1} = G\mathbold{x}_{i+1} 
    \vspace{-1mm}
\end{align}
where $\mathbold{x}_{i} \in \mathbb{R}^{n_{\mathbold{x}}}$, $\mathbold{u}_{i} \in   \mathbb{R}^{n_{\mathbold{u}}}$, and $\mathbold{y}_{i} \in \mathbb{R}^{n_{\mathbold{y}}}$ for all $i \in \mathbb{N}$ correspond to the system state, controls, and output at time index $i$ respectively. The state transition, actuation, and measurement matrices are $A \in \mathbb{R}^{n_{\mathbold{x}}\times n_{\mathbold{x}}}  $, $B\in \mathbb{R}^{n_{\mathbold{x}}\times n_{\mathbold{u}}}$, and $G\in \mathbb{R}^{n_{\mathbold{y}}\times n_{\mathbold{x}}}$ respectively. Let $C\in \mathbb{R}^{3\times n_{\mathbold{x}}}$ be a constant matrix that maps the system's states \eqref{eq:DT_system_dynamics} to the system's Cartesian $x,y,z$ positions with respect to an inertial frame $E$.
We model the $k^{\text{th}}$ obstacle, $k \in \mathbb{Z}^{1:N^{\text{obs}}}$, as a sphere. The set of Cartesian points occupied by the obstacle is $\mathcal{O}_{k}(\mathbold{c}_{k}, r^k) = \{ \mathbold{x} \in \mathbb{R}^{3} : \|\mathbold{c}_{k}-\mathbold{x} \|_2 \leq r_{k}\}$, where $\mathbold{c}_{k} \in \mathbb{R}^{3}$ and $r^k \in \mathbb{R}_+$ are the center and radius of the $k^{\text{th}}$ obstacle. 

This paper considers the case where the agent \eqref{eq:DT_system_dynamics} is tasked with following a specified  reference output trajectory $\mathbold{y}^{\text{ref}}$, whose geometry need not incorporate any obstacle information.  While following this path, the agent may encounter $N^\text{obs}$ spherical stationary or moving obstacles. The obstacle-free region is given by the open set:
\begin{equation} \label{eq:obs_free_set}
    \mathcal{S} \triangleq \left\{\mathbb{R}^{3} \setminus \cup_{k = 1}^{N^\text{obs}} \mathcal{O}_{k}\right\} .
\end{equation}
\begin{assumption} \label{assum:assumption1}
Obstacles can be detected and localized at the same rate ($f^{+}$ Hz) of the planner update. Only measurements of an obstacle's geometric center with respect to frame E are assumed, and they are corrupted by a zero-mean noise. We can estimate the radius, $r_k$, of the $k^{\text{th}}$ obstacle as $\hat{r}_k$, and the estimate satisfies $\hat{r}_k \geq r_k$. \footnote{It is important to note that Assumption \ref{assum:assumption1} \textbf{does not} imply full state measurement. See Fig. \ref{fig:MSSA_architecture} for an example of Frisbee. } 
\vspace{-1mm}
\end{assumption}

%
\begin{assumption} \label{assum:assumption2}
All obstacle measurements, admit an L-decomposition of order $d$, are governed by LRFs \eqref{eq:LRF} whose LRF coefficients can be uniquely defined.
\vspace{-1mm}
\end{assumption}
%
\begin{assumption} \label{assum:assumption3}
We assume that the obstacles' velocities are bounded by $v_{\text{max}}$, and the initial displacements between all obstacles and the agent are  significantly greater than $ \frac{dv_{\text{max}}}{f^+}$. 
\vspace{-1mm}
\end{assumption}


\begin{problem}\label{problem:formulation0}\textbf{[Prediction]} Consider a multivariate stochastic process where observables $\{x\}_{1}^{N}$, $\{y\}_{1}^{N}$, and $\{z\}_{1}^{N}$ correspond to the spherical obstacle's true center location with respect to a common reference frame, E. 
The measurements are corrupted by independent, zero-mean noises $\{\gamma_{1}\}_{1}^{N}$, $\{\gamma_{2}\}_{1}^{N}$, and $\{\gamma_{3}\}_{1}^{N}$ (see Fig. ~\ref{fig:MSSA_architecture}). Under Assumptions \ref{assum:assumption1}-\ref{assum:assumption3}, we seek to predict the obstacle position at times $N+1$ to $N+N^h$ using measurements where $N^h \in \mathbb{Z}_+$.
\vspace{-2mm}
\end{problem}
Due to limited and noisy partial data and the lack of explicit dynamics models, 
we estimate a Bootstrap distribution of the obstacle predictions, denoted by the random set $\mathcal{O}^{\text{pred}}$, from time index $N+1$ to $N+N^h$ and calculate its first and second moments. We  account for errors in the forecast locations due to poor signal and noise separation and bandwidth limits (due to limited training data and incorrect choices of embedding length $L$) by solving  a distributionally robust chance constrained model predictive planning problem. 



\begin{problem}\label{problem:formulation1}\textbf{[Planning]} Consider the system \eqref{eq:DT_system_dynamics} and free-space \eqref{eq:obs_free_set}. Given a discrete-time reference trajectory $\mathbold{y}^{\text{ref}}_i \, \forall i \in \mathbb{Z}^{1:N^h}$ where $N^h \in \mathbb{Z}_{+}$ is the length of the horizon, convex state constraints $\mathcal{D}^{\mathbold{x}}\subset \mathbb{R}^{n_{\mathbold{x}}}$, convex input constraints $\mathcal{D}^{\mathbold{u}}\subset \mathbb{R}^{n_{\mathbold{u}}}$, and a convex stage cost function $L_i: \mathbb{R}^{n_{\mathbold{x}}}\times  \mathbb{R}^{n_{\mathbold{u}}} \to \mathbb{R}_{\geq0}$, a total of $N^{\text{obs}}$ spherical obstacles each approximated by a set $\mathcal{O}_k^{\text{pred}}$, and risk tolerance $\epsilon\in (0,1]$, we seek to compute a receding horizon controller $\{ \mathbold{u}^{*}\}_{1}^{N_h}$ that avoids the unsafe set ${\mathcal{O}^{\text{pred}}} \triangleq \bigcup_{k=1}^{N^{\text{obs}}}\mathcal{O}_{k}^{\text{pred}}$ via the following non-convex optimization problem:
\vspace{-2mm}
\begin{subequations}
\begin{align}
        \{ \mathbold{u}^{*}\}_{1}^{N_h} = &\min_{\tiny \begin{array}{c}
           \{\mathbold{u}_k\}^{N^h}_1\in \mathbb{R}^{n_{\mathbold{u}}} 
        \end{array}}\normalsize   \sum_{i=1}^{N_h} L_i(\mathbold{y}^{ref}_i - \mathbold{y}_i,\mathbold{u}_i) \\
     \mbox{s.t.} \quad&   \mathbold{x}_{i+1}\! =\!  A\mathbold{x}_{i}\!+\!B\mathbold{u}_{i} \quad \mathbold{y}_{i+1} \!=\! G\mathbold{x}_{i+1}\\
     & \mathbold{x}_{i} \in \mathcal{D}^{\mathbold{x}}, \quad \mathbold{u}_{i} \in \mathcal{D}^{\mathbold{u}}\label{eq:optimization_stateninput_constraint}, \quad \mathbold{x}_1 = \mathbold{x}_{init}\\
     &\mathbb{P}(\mathbold{x}_i \in {\mathcal{O}^{\text{pred}}}) \leq \epsilon \label{eq:optimization_obsavoid_setconstraint}, \quad \forall i \in\mathbb{Z}^{1:N^h}
\end{align}
\end{subequations}
\end{problem}
\vspace{-2mm}

\section{Bootstrap Forecasting}
\vspace{-2mm}\label{sec:bootstrap}


Despite empirical successes in reconstructing and forecasting \cite{ssa_withR}, the theoretical accuracy of SSA is strenuous to obtain, see \cite{forecast_errorbounds}. Inspired by \cite{ensemble_sindy}, we use bootstrapping to improve model discovery and to produce probabilistic forecasts.
\vspace{-2mm}
\begin{algorithm}[!h]
\footnotesize
\caption{Bootstrap Forecast Algorithms (Per Obstacle)}\label{alg:one}
\KwData{Obstacle  center position measurements $\{\hat{\mathbold{x}}\}_{1}^{N},\{\hat{\mathbold{y}}\}_{1}^{N},\{\hat{\mathbold{z}}\}_{1}^{N}$, \newline 
User defined constants: $N^{\text{train}}$,$N^{\text{step}}$, $\delta_t$, $N_\sigma$, $N^{\text{strap}}$ }
\KwResult{Forecast:$\{^j\mathbold{x}\}_{N+1}^{N+N^h}\!,\!\{^j\mathbold{y}\}_{N+1}^{N+N^h}\!,\!\{^j\mathbold{z}\}_{N+1}^{N+N^h}\!,\!\forall j\in \mathbb{Z}^{1:N^\text{straps} }$}
Use  $\{\hat{\mathbold{x}}_{N+1},\hat{\mathbold{y}}_{N+1}$,$\hat{\mathbold{z}}_{N+1}\}$ to update Hankel matrix\;\newline
\While{$\text{istrap} \leq
N^\text{strap}$}{
\While{ $N+1 \geq N^{\text{train}}$}{
\For{$\text{states} = x,y,z$}{
  \While{\scriptsize  $\|Y_{N+1}^{\lambda_1:\lambda_t} - Y_{N+1}^{\lambda_1:\lambda_{t+1}}\|_2$\newline $- \|Y_{N+1}^{\lambda_1:\lambda_{t+1}} - Y_{N+1}^{\lambda_1:\lambda_{t+2}}\|_2 \geq \frac{\delta_t}{N+1}$\normalsize}{
    $t = t+1$}{
  obtain the tuple for each $\text{states}$:($\{\lambda^{\text{istrap}}\}_{1}^{t}$, $\{\mathbold{\mu}^{\text{istrap}}\}_{1}^{t}$, $\mathbold{\phi}^{\text{istrap}}$)}, $\quad \text{istrap} =  \text{istrap}+1$\;
  \newline
  \For{$tt = t+1:t+N^{\sigma}$}{
  obtain the tuple for $\text{states}$: ($\{\lambda^{\text{istrap}}\}_{1}^{tt}$, $\{\mathbold{\mu}^{\text{istrap}}\}_{1}^{tt}$, $\mathbold{\phi}^{\text{istrap}}$), $\quad \text{istrap} =  \text{istrap}+1$}
}
    $N^{\text{train}} = N^{\text{train}} +N^{\text{step}}$
}
{Back-up Strategy}
}{
Apply the tuples ($\{^j\lambda^{\text{istrap}}\}_{1}^{t_j}$, $\{^j\mathbold{\mu}^{\text{istrap}}\}_{1}^{t_j}$, $^j\mathbold{\phi}^{\text{istrap}}$) $\forall j\in\mathbb{Z}_{1:\text{Nstraps}}$ for $x,y,z$ to the updated Hankel, where $t_j$ denotes number of eigenvalues post truncation for the $j^{th}$ bootstrap. Perform a $N^h$ step forecast using $^j\mathbold{\phi}^{\text{istrap}}$.
}
\normalsize
\end{algorithm}
\vspace{-3mm}

Our real-time bootstrap forecast, Algorithm~\ref{alg:one}, assumes time series measurements of the form \eqref{eq:signal_struc}. The user-defined parameters $N^{\text{train}}$ and $N^{\text{step}}$  represent the allowed number of initial training samples, and the number of newly accumulated samples during an initial bootstrap. Further, one must choose parameters $\delta_{t}$ and $N_{\sigma}$, where $\delta_{t}$ is the threshold used to separate signal from noise, and $N_{\sigma}$ is the number of steps of progressive relaxation of threshold $\delta_{t}$.\footnote{The parameters $\delta_{t}$ and $N_{\sigma}$ are dictated by measurement noise levels, which can be characterized off-line in a controlled experimental setting.  
} Recall the desired signal/noise separation \eqref{eq:noise_signal_separation}, the theoretical optimal threshold $\mathbold{\lambda}$ is unknown and must be estimated. Let $Y_{N}^{\lambda_1:\lambda_d}$ be the Hankelization reconstructed $\hat{\mathbold{y}}$ with the eigenvalues $\{\lambda\}_{1}^{d}$ and their corresponding right and left eigenvectors. Note, if $d > n$ where~\small $\lambda_n \leq \mathbold{\lambda}\leq \lambda_{n+1}$, then the norm values  $\|Y_{N}^{\lambda_1:\lambda_{d+t}} - Y_{N}^{\lambda_1:\lambda_{d+t+1}}\|_2 \approx \|Y_{N}^{\lambda_1:\lambda_{d+t+1}} - Y_{N}^{\lambda_1:\lambda_{d+t+2}}\|_2$\normalsize~since they are comprised of the residual measurement noise. We threshold the difference between two consecutive reconstructions with $\delta_t/N$, i.e. finding the smallest $t\in \mathbb{Z}_+$ s.t.:
\begin{equation} \label{eq:threshold_sigma}
    \|Y_{N}^{\lambda_1:\lambda_t} - Y_{N}^{\lambda_1:\lambda_{t+1}}\|_2 - \|Y_{N}^{\lambda_1:\lambda_{t+1}} - Y_{N}^{\lambda_1:\lambda_{t+2}}\|_2 \leq \frac{\delta_t}{N}
    \vspace{-1mm}
\end{equation}
Since the selection of the threshold $\delta_t$ is crucial, we add an additional parameter $N^{\sigma}$ to ensure no principle components are lost in $Y_{N}^{\lambda_1:\lambda_d}$ because of bad choice of $\delta_t$, i.e. to avoid $d < n$. To be conservative, we include the next $N^{\sigma}$ largest eigenvalues after the first $t$ eigenvalues in the bootstrapping process. Most importantly, the number of bootstraps, $N^{\text{strap}}$, needs to be determined \textit{a priori}, considering the computation capacity, number of obstacles, and the expected noise level.

The effectiveness of Algorithm \ref{alg:one} depends highly on the time delay length $L$, the number of training measurements $N^{\text{train}}$, the number of bootstraps $N^{\text{strap}}$, and the MPC horizon length, $N^{h}$. We recommend that $N^{\text{train}}$ be at least $10N^h$ and that $L = \frac{N^{\text{train}}}{4}$. $N^{\text{strap}}$ and $N^\text{step}$ should be as large as allowed by the computing platform and benchmarking them offline. 


\vspace{-2mm}
\section{Bootstrap Planning}
\vspace{-2mm}

This section introduces an MPC-based path planner to solve  Problem \ref{problem:formulation1}. First, we revisit the obstacle avoidance constraint \eqref{eq:optimization_obsavoid_setconstraint} and its properties given the mean and variance of the bootstrap predictions. Next, we use this obstacle avoidance constraint in the MPC optimization, and provide probabilistic guarantees of constraint satisfaction. 
Algorithm \ref{alg:one} produces $N^{\text{strap}}$ copies of $N^h$ length predictions of the $k^{{\text{th}}}$ obstacle's location. We denote the $j^{\text{th}}$ copy of the bootstrap prediction as \small$
    \{\hat{\mathbold{y}}^{j}_k\}_{1}^{N_{h}} =
    \{
    \hat{\mathbold{y}}^{j}_{1,k}, \hat{\mathbold{y}}^{j}_{2,k}, \cdots, \hat{\mathbold{y}}^{j}_{N^h,k}\}
$\normalsize. The collision avoidance set constraint \eqref{eq:optimization_obsavoid_setconstraint} can be reformulated based on the obstacle shape and center as $\|C\mathbold{x}_{i} -\! \hat{\mathbold{y}}^j_{i,k}\|_2 \geq \hat{r}_k + r_p \triangleq \overline{r}_k$, for each $i \in \mathbb{Z}^{1:N^{h}}$ and $k \in \mathbb{Z}^{1:N^{\text{obs}}}$ and $r_p$ is the safety radius of the agent $\eqref{eq:DT_system_dynamics}$.
This constraint can be equivalently expressed as the following concave (in the state $\mathbold{x}_i$) constraint,
\vspace{-0.25mm}
\begin{align} \label{eq:convexification0}
   (C\mathbold{x}_{i}\!\,-\, \hat{\mathbold{y}}^j_{i,k})^{T}(C\mathbold{x}_{i}\,\,- \,\hat{\mathbold{y}}^j_{i,k})\,\geq\, \overline{r}_k \|(C\mathbold{x}_{i}\! - \!\hat{\mathbold{y}}^j_{i,k})\|_2.
   \vspace{-1mm}
\end{align}
We approximate \eqref{eq:convexification0} as an affine constraint through the use of Sequential Convex Programming (SCP)~\cite{SCP1,scp2}
\vspace{-0.25mm}
\begin{align} \label{eq:convexification1}
   (C\mathbold{x}_{i} - \,\hat{\mathbold{y}}^j_{i,k})^T(C\overline{\mathbold{x}}_{i} - \,\hat{\mathbold{y}}^j_{i,k}) \geq \overline{r}_k \|(C\overline{\mathbold{x}}_{i}  - \,\hat{\mathbold{y}}^j_{i,k})\|_2
   \vspace{-1mm}
\end{align}
where $\overline{\mathbold{x}}_{i}$ is approximated with the solution from previous SCP iterations. Note that Eq. \eqref{eq:convexification1}  over-approximates constraint ~\eqref{eq:convexification0} (see~\cite{SCP1} for proof). 
\begin{lem} \label{lemma:lemma1}
\vspace{-1mm}

If we have $N^{\text{strap}}$ forecasts of the  $k^{\text{th}}$ obstacle's position from time $i\in \mathbb{Z}^{1:N^{h}}$ and the previous SCP trajectory $\{\overline{\mathbold{x}}\}_{1}^{N_h}$, then we can define the $j^{th}$ bootstrap lumped collision avoidance coefficients $\alpha_{i,k}^j$, $\beta_{i,k}^j$ and the standard deviation of the collision avoidance constraint $\Delta_{i,k}$ as:
\vspace{-1mm}
\begin{align}
    \alpha_{i,k}^j &\triangleq -C^{T}(C\overline{\mathbold{x}}_i -\, \hat{\mathbold{y}}^j_{i,k})\label{eq:alpha}\\
    \beta_{i,k}^j &\triangleq \overline{r}_k\|(C\overline{\mathbold{x}}_{i}  -\, \hat{\mathbold{y}}^j_{i,k})\|_2- (C\overline{\mathbold{x}}_{i})^T(C\overline{\mathbold{x}}_{i}  - \, \hat{\mathbold{y}}^j_{i,k})\label{eq:beta} \\
    \Delta_{i,k} &\triangleq \sqrt{\mathbold{p}_i^{T}\Sigma_{\alpha_{i,k}}\mathbold{p}_i + 2\mathbold{p}_k^{T}\Sigma_{\alpha\beta_{i,k}} + \Sigma_{\beta{i,k}}} \label{eq:delta_ik},
    \vspace{-1mm}
\end{align}
where, \small$\Sigma_{\alpha_{i,k}}\triangleq \mbox{cov}\left(\alpha_{i,k}^j,\alpha_{i,k}^j\right)$\normalsize, \small$\Sigma_{\beta_{i,k}}\triangleq \mbox{cov}\left(\beta_{i,k}^j,\beta_{i,k}^j\right)$\normalsize, and \small$\Sigma_{\alpha\beta_{i,k}}\triangleq \mbox{cov}\left(\alpha_{i,k}^j,\alpha_{i,k}^j\right)$\normalsize~are sample covariance matrices computed using the bootstrapped coefficients ~\small$\{\alpha_{i,k}\}_{1}^{N^{\text{strap}}}$\normalsize~and~\small$\{\beta_{i,k}\}_{1}^{N^{\text{strap}}}$\normalsize~and~\small$\mathbold{p}_i \triangleq C\mathbold{x}_i \in \mathbb{R}^{3}$\normalsize~. Let the dimension of the null space of~\small$\Sigma_{\alpha_{i,k}}$ be $n_{i,k} \geq 0$\normalsize~.\footnote{For all our numerical simulation, $\Sigma_{\alpha_{i,k}}$ is strictly positive definite. However, in the case of one or multiples measurable states are noiseless,  $\Sigma_{\alpha_{i,k}}$ can be ill-conditioned. Alternative to adding $I^{\text{null}}_{i,k}$ which can be numerically expansive to determine, we recommend applying Algorithm~\ref{alg:one} only to states that measurement noises are present and adapt Theorem \ref{thm:formulation3} with deterministic forecasts for the states without noise and the distributionally robust chance constraint formulation for the noisy ones.} The standard deviation $\Delta_{i,k}$ has the following upper bound,
\vspace{-1mm}
\begin{align}
    \Delta_{i,k}\! \leq\! \mathbold{1}^{T}|\Tilde{\Sigma}_{\alpha_{i,k}}^{\frac{1}{2}}\left(\mathbold{p}_i\!-\!\mathbold{h}_{i,k}\right)|\! +\!\! \sqrt{3k_{i,k}}\triangleq \zeta_{i,k},
    \vspace{-2mm}
\end{align}
where~\small$\Tilde{\Sigma}_{\alpha_{i,k}}\! =\! \Sigma_{\alpha_{i,k}}\!+I_{i,k}^{\text{null}}$, $I_{i,k}^{\text{null}} = \left[\begin{smallmatrix}
\mathbold{0} & \mathbold{0}\\
\mathbold{0} & I_{n_{i,k}\times n_{i,k}}
\end{smallmatrix}\right]\in \mathbb{R}^{3\times3}$\normalsize~, and
\vspace{-2mm}
\begin{align} \label{eq:handk}
\left[\begin{smallmatrix}
\mathbold{h}_{i,k}\\
 k_{i,k}
\end{smallmatrix} \right] &\triangleq \left[\begin{smallmatrix}
-\left(\Sigma_{\alpha_{i,k}} +I_{i,k}^{\text{null}}\right)^{-1}\Sigma_{\alpha\beta_{i,k}}\\
\Sigma_{\beta_{i,k}}-\Sigma_{\alpha\beta_{i,k}}^{T}\left(\Sigma_{\alpha_{i,k}}+ I_{i,k}^{\text{null}}\right)^{-1}\Sigma_{\alpha\beta_{i,k}}
\end{smallmatrix}
\right].\end{align} 
\end{lem}
\vspace{-3mm}
\begin{proof}
See Appendix.
\end{proof}
\vspace{-2mm}
\begin{figure*}[!ht]
    \centering
    \includegraphics[width=17.5cm]{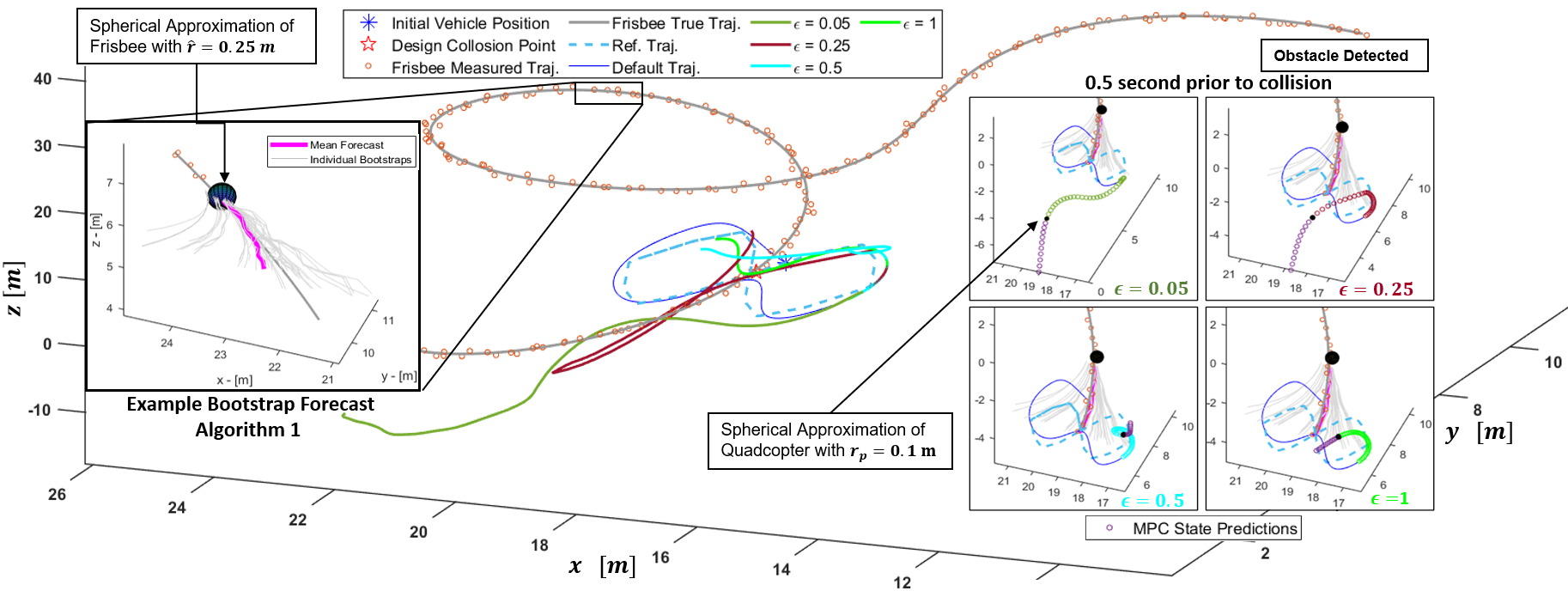}
    \label{fig:data}
        \caption{Four Monte-Carlo simulations with agent dynamics~\eqref{eq:sim_agent_dyn} and a Frisbee obstacle~(see Fig.\ref{fig:MSSA_architecture}) are compared. The same obstacle behaviors are simulated while the agent tracks the same figure '8' reference trajectory with four risk levels $\epsilon = \{0.05,0.25,0.5,1\}$. The simulation is designed to be difficult: the vehicle must deviate from its reference trajectory as the obstacle trajectory is designed to intersect the agent's reference trajectory with noise obstacle trajectory measurements. All measurement noises are sampled uniformly between $[-0.125,\, 0.125]$ meters. The bootstrap obstacle forecast uses the parameters: $L = 24$, $N^{\text{train}} = 100$, $N^{\text{step}} = 5$, $\delta_t = 20$, $N_\sigma = 8$, $N^{\text{strap}} = 40$. SSA-MPC uses the constants $N^h = 10$, $\chi = 50$ and $\tau =0.25$ with fixed 4-step SCP iterations. The tuple ($\{\lambda^j\}_{1}^{t_j}$, $\{\mathbold{\mu}^j\}_{1}^{t_j}$, $\mathbold{\phi}^j$),$\forall j \in \mathbb{Z}^{1:40}$ in Algorithm \ref{alg:one} is computed with observables measured at 20 Hz. The four sub-diagrams show the planned trajectory at 4 risk levels; the planner is more conservative as $\epsilon \to 0$, and aligns with the results shown in Table~\ref{table:table1}.}
        \vspace{-4mm}
\end{figure*}
Note that while all bootstraps can be incorporated as separate obstacle avoidance constraints, such operations are costly, as the number of constraints grows linearly with $N^{\text{strap}}$. Instead, we estimate the ensemble mean and covariance of the distance from the obstacle. The theorem below uses a distributionally robust chance constraint to account for  \textit{all bootstrap distributions} that can have this mean and covariance. This approach results in a significantly reduced number of obstacle avoidance constraints, and this number remains fixed regardless of the number of bootstrap predictions $N^{\text{strap}}$.

\begin{theorem}
(SSA-MPC) \label{thm:formulation3}
Consider Problem \ref{problem:formulation1} under Assumptions \ref{assum:assumption1}-\ref{assum:assumption3} with system dynamics $\eqref{eq:DT_system_dynamics}$ and bootstrap SSA forecasts of all obstacles' center positions. Given a risk tolerance $\epsilon$, the solution to the following optimal control problem  is a feasible solution of Problem \ref{problem:formulation1} as $w\xrightarrow[]{}\infty$. 
The SCP optimization problem at iteration $w$ is:
\vspace{-3mm}
\begin{subequations}\label{eq:optimization_risk}
\begin{align} 
        \{ \mathbold{u}^{*}\}_{1}^{N_h} &=  \min_{\tiny \begin{array}{c}
           \mathbold{u}_{i}\in \mathbb{R}^{n_{\mathbold{u}}}  \\
            \mathbold{s}_{i,k}\in \mathbb{R}^{3}
        \end{array}}\normalsize   \sum_{i=1}^{N_h} L_i(\mathbold{y}^{ref}_{i} - G\mathbold{x}_{i},\mathbold{u}_i) \\
     \mbox{s.t.} \quad &  \mathbold{x}_{i+1} =  A\mathbold{x}_i+ B\mathbold{u}_i\\
     & \mathbold{x}_{i} \in \mathcal{D}^{\mathbold{x}}, \quad \mathbold{u}_{i} \in \mathcal{D}^{\mathbold{u}}, \quad \mathbold{x}_1 = \mathbold{x}_{init} \label{eq:optimization_stateninput_scp_constraint} \\
     & \Lambda_{i,k}\begin{bmatrix}
\mathbold{x}_i &
\mathbold{s}_{i,k}
\end{bmatrix}^{T} \leq \Gamma_{i,k}
 \label{eq:nonlinear_robust_scp_constraint_risk}\\
     & \|\mathbold{x}_{i} - \overline{\mathbold{x}}_i\| \leq \chi \tau^{w} \label{eq:nonlinearSCP_guarantees}\quad  \forall i,k \in \mathbb{Z}^{1:N^h}\! \times\! \mathbb{Z}^{1:N^{\text{obs}}}
\end{align}
\end{subequations}
where $\{\bar{\mathbold{x}}\}_{1}^{N^{h}}$ is the solution to the $(w-1)^{\text{th}}$ iteration of the SCP optimization, $\Lambda_{i,k}\in \mathbb{R}^{7\times11}$ and $\Gamma_{i,k}\in \mathbb{R}^{7}$ encode the risk-based collision avoidance relationships,
\vspace{-1mm}
\begin{align*}
    \Lambda_{i,k}\! =\! \left[\begin{smallmatrix}
    \mathbb{E}[\alpha_{i,k}]^{T}C & \mathbold{1}^{T}\nu_{\epsilon_n}\\
    \Tilde{\Sigma}_{\alpha_{i,k}}^{1/2}C & -I_{3\times3} \\
    - \Tilde{\Sigma}_{\alpha_{i,k}}^{1/2}C & -I_{3\times3} \\
    \end{smallmatrix}\right], \quad  \Gamma_{i,k}
\!= \!\left[\begin{smallmatrix}
    -\mathbb{E}[\beta_{i,k}]- \nu_{\epsilon_n}\sqrt{3k_{i,k}}\\
     \Tilde{\Sigma}_{\alpha_{i,k}}^{1/2}\mathbold{h}_{i,k}\\
    -\Tilde{\Sigma}_{\alpha_{i,k}}^{1/2}\mathbold{h}_{i,k}
    \end{smallmatrix} \right].
\end{align*}
such that~\small$\epsilon_n \triangleq \frac{\epsilon}{N^\text{obs}}$\normalsize~and~\small $\nu_{\epsilon_n} \triangleq \sqrt{\frac{1-\epsilon_n}{\epsilon_n}}$\normalsize.
Lastly,~\small$\chi \geq 0$\normalsize~and~\small$\tau \in (0,1)$\normalsize~are the initial trust region and worst-case rate of convergence, respectively.
\vspace{-3mm}
\end{theorem} 
\begin{proof}
The $j^{\text{th}}$ random bootstrapped obstacle forecasts can be denoted as $   \mathbold{z}_{i,k}^{j} \triangleq (\alpha^j_{i,k})^{T} \mathbold{x}_{i} + \beta^j_{i,k}$, where $\alpha^j_{i,k}$ and $\beta^j_{i,k}$ are defined in \eqref{eq:alpha} and \eqref{eq:beta}. We have shown that the obstacle avoidance constraint~\eqref{eq:convexification0} has an affine over approximation~\eqref{eq:convexification1}, which is equivalently given by $\mathbold{z}_{i,k}^{j} < 0$. Hence, the chance constraint~\eqref{eq:optimization_obsavoid_setconstraint} is, 
\begin{align*}
\vspace{-1mm}
    \mathbb{P}(\mathbold{x}_{i} \in \mathcal{O}_{\text{pred}}) = \mathbb{P}\Big(\bigcup_{k=1}^{N^{\text{obs}}} \{\mathbold{z}_{i,k} \geq 0\}\Big) \leq \sum_{k=1}^{N^{\text{obs}}} \mathbb{P}(\mathbold{z}_{i,k}\geq 0).
    \vspace{-1mm}
\end{align*}
Enforcing the chance constraints $\mathbb{P}(\mathbold{z}_{i,k} \geq 0) \leq \epsilon_n$, $\forall k \in \mathbb{Z}^{1:N^{\text{obs}}}$ also satisfies~\eqref{eq:optimization_obsavoid_setconstraint}.
We can satisfy this chance constraint in a distributionally robust manner:
\begin{equation*}
    \sup_{\mathbold{\kappa} \sim\left(\mathbb{E}[\mathbold{z}_{i,k}],\Sigma_{\mathbold{z}_{i,k}}\right)}\mathbb{P}\{\mathbold{\kappa}\geq0 \} \leq \epsilon_n, \, \forall i,k \in \mathbb{Z}^{1:N^{h}} \times \mathbb{Z}^{1:N^{\text{obs}}},
\end{equation*}
where~\small$\mathbb{E}[\mathbold{z}_{i,k}]$\normalsize~and~\small$\Sigma_{\mathbold{z}_{i,k}}$\normalsize~are the sample mean and covariance matrix of the bootstrapped~\small$\{\mathbold{z}_{i,k}\}_{1}^{N^\text{strap}}$\normalsize. We reformulate the above statement as a deterministic constraint as shown in~\cite{ghaoui2003DRO},
\begin{equation} \label{eq:inequality_risk_inZ}
    \underbrace{\mathbb{E}[\mathbold{z}_{i,k}]}_{\mathbb{E}[\alpha_{i,k}]^{T}C\mathbold{x}_i+ \mathbb{E}[\beta_{i,k}]}\!\!\!\!\!\!\!\!\!\!\!\!\!\! + \nu_{\epsilon_n}\!\underbrace{\sqrt{\Sigma_{\mathbold{z}_{i,k}}}}_{\Delta_{i,k}}\! \leq\! 0,\,   \forall i\! \in\! \mathbb{Z}^{1:N^h}\!, k\! \in\! \mathbb{Z}^{1:N^{\text{obs}}}.
\end{equation}
Constraint (\ref{eq:inequality_risk_inZ}) is not affine in the optimization variable, as is desirable for real-time application. 
By Lemma \ref{lemma:lemma1}, $\Delta_{i,k}\leq \zeta_{i,k}$, and we deduce the following tighter inequality constraint as a numerically appealing alternative to~\eqref{eq:inequality_risk_inZ},\vspace{-1mm}
\begin{multline}\label{eq:scp_constraint_risk}
     \mathbb{E}[\alpha_{i,k}]^{T}C\mathbold{x}_i + \mathbb{E}[\beta_{i,k}] +\\ \nu_{\epsilon_n}\left(\mathbold{1}^{T}|\Sigma_{\alpha_{i,k}}^{1/2}\mathbold{p}_i-\mathbold{h}_{i,k}|+ \sqrt{3k_{i,k}}\right)\leq 0.
     \vspace{-2mm}
\end{multline}
To account for the absolute value term, we introduce auxiliary optimization variables $\mathbold{s}_{i,k}$ that satisfy the following:
\begin{align*}
    \Sigma_{\alpha_{i,k}}^{1/2}\mathbold{p}_i-\mathbold{h}_{i,k} \leq \mathbold{s}_{i,k} , \quad -\Sigma_{\alpha_{i,k}}^{1/2}\mathbold{p}_i+\mathbold{h}_{i,k} \leq \mathbold{s}_{i,k}.
    \vspace{-2mm}
\end{align*}
Therefore, satisfying~\eqref{eq:nonlinear_robust_scp_constraint_risk} is equivalent to satisfying~\eqref{eq:scp_constraint_risk}.

Convergence of the SCP is proven in \cite{nonlinearSCP} which is based on implementing a trust region via second-order cone constraints \eqref{eq:nonlinearSCP_guarantees}. The authors also show the solution to the SCP formulation as $w\to \infty$ feasibly solves problem \ref{problem:formulation1}.\footnote{To be numerically feasible, $w$ is usually upper bounded by a finite integer, resulting in a sub-optimal but still feasible solution.} 
\\
\vspace{-8mm}
\end{proof}


\begin{table}
\begin{center}
\footnotesize
\begin{tabular}{ |c|c|c|c|c|c|c| } \hline 
  Cases& $\epsilon$ &  $0.05$& $0.1$ & $0.25$ &$0.5$ & $1$ \\ \hline \hline
 & $\%$Feas.                   & 97.5 & 98.2  & 98.9 & 99.6& 100\\ 
Const. & $\%$Succ.             & 100 & 100  & 100 & 100& 59.0\\ 
Speed & $\overline{d}_{min}$   & 2.26 & 1.85 & 1.41& 1.12 & 0.64\\ \
&$\sigma(d_{min})$             & 0.42 &  0.33 & 0.25& 0.22& 0.35\\ \hline  \hline
  & $\%$Feas.                  & 99.5      & 99.6     & 99.9 & 100    & 100 \\
Ball & $\%$Succ.               & 100      & 100    & 100 & 100    & 79.3\\ 
w/drag &$\overline{d}_{min}$   & 2.60      & 2.14   & 1.63 & 1.27     & 0.64\\ 
&$\sigma(d_{min})$             & 1.08      & 0.93   & 0.70 & 0.50   & 0.27\\ \hline  \hline
 & $\%$Feas.                   & 90.3  & 97.4   & 98.3 & 98.6 & 97.8\\ 
Frisbee & $\%$Succ.            & 100  & 100     &100   & 100 & 58.0\\ 
w/drag & $\overline{d}_{min}$  & 4.97  & 3.97   &2.85 & 2.01 &  0.78\\ \
&$\sigma(d_{min})$             & 1.97  & 1.53   &1.15 & 0.91 &  0.77\\ \hline 
\end{tabular}
\caption{Summary of results from Monte-Carlo simulations.}
\label{table:table1}
\normalsize
\end{center}
\vspace{-6mm}
\end{table}
\vspace{-1mm}
\section{Numerical Example}
\vspace{-2mm}
We consider a quadcopter that follows a reference trajectory $\mathbold{y}^{\text{ref}}$ while avoiding randomly generated moving obstacles and adhering to state and control constraints. Let the position of the quadcopter in frame $E$ be $x, y, z$ and the Euler angles roll, pitch, and yaw are given by $\varphi, \theta, \psi$ respectively. The following dynamic model is used in the simulation: 
\begin{equation} \label{eq:sim_agent_dyn}
\vspace{-1mm}
 \ddot{x} = -9.81\theta, \, \,
       \ddot{y} = 9.81\varphi,\,\, \ddot{z}=-u_1 - 9.81,\,\, \ddot{\psi} = u_4,
\end{equation}
The planner control inputs are given by $u_1, \theta, \varphi, u_4$. The reference trajectory consists of the desired positions, \small$\{x^{\text{ref}}\}_{1}^{N^{T}},\{y^{\text{ref}}\}_{1}^{N^{T}},\{z^{\text{ref}}\}_{1}^{N^{T}}$\normalsize~and yaw angles \small$\{\psi^{\text{ref}}\}_{1}^{N^{T}}$\normalsize.

To demonstrate the effectiveness of the proposed method, we conducted Monte-Carlo (MC) simulations of the proposed planner avoiding three differently behaved obstacles which are introduced once in each run.  See the provided simulator for details.\footnote{https://github.com/skylarXwei/Riskaware$\_$MPC$\_$SSA$\_$Sim.git} Case 1 is a constant speed spherical obstacle without drag. Case 2 is a thrown spherical (ball) obstacle with drag. Case 3 is a Frisbee that is thrown at various initial angles, position, speed, and rotation speed. The spherical obstacle dynamics are captured by a 6 state ODE with drag penalties proportional to its velocities. The Frisbee is modelled following \cite{hummel2003frisbee}, using a full 12 state model identical to Fig.\ref{fig:MSSA_architecture} and aerodynamic drag coefficients. 

We conduct 1000 MC simulations per $\epsilon$ level to compare the numerical feasibility, percent success in obstacle avoidance (if the MPC planner is feasible), and the planner's conservativeness, as measured by the minimum distance between the obstacle and agent centers. For the three cases, the obstacle speed ranges are $[0.41,\, 8.43],\, [3.41,\, 6.37],$ and $[5.76,\,6.68]$ m/s, respectively.  The MPC planning and measurement rates are fixed to be 20 Hz. With a 10 step horizon length and 40 bootstraps, the average per planner update rate is $0.030\pm 0.0014$ sec solved using Gurobi~\cite{gurobi} on an Intel i7-9700K CPU @3.6GHz processor, dynamic simulation written in MATLAB. 
The results in Table~\ref{table:table1} show the applicability of our SSA-MPC algorithm, despite vast differences in obstacle behavior. Further, as the risk tolerance $\epsilon$ shrinks, the percentage success in obstacle avoidance (when the solution is feasible) increases, with a trade-off in the feasibility of optimization  \eqref{eq:optimization_risk}. The risk tolerance $\epsilon$ can also viewed as a robustness parameter which inversely proportional to the distance between the agent and obstacles. 
\vspace{-2mm}
\section{Conclusion}
\vspace{-2mm}
Our data-driven risk-aware obstacle avoidance planner showcased near perfect results in avoiding moving obstacles with limited and noisy measurements and no prior knowledge about the obstacle behaviors. We not only offered a new paradigm that can extract obstacle dynamics online allowing short prediction, but an equally important risk-aware MPC formulation that enables real-time usage. The simulation result also shows that adjusting the risk level $\epsilon$ can implicitly adjust the safety distance between the agent and obstacles.




\vspace{-2mm}
\footnotesize{
\bibliography{references}
}
\vspace{-1mm}
\bibliographystyle{ieeetr}
\normalsize

\vspace{-3mm}
\section*{APPENDIX}
\vspace{-2mm}
\begin{proof}


Let the eigendecomposition of ~\small$\Sigma_{\alpha_{i,k}}$\normalsize~be the following:~\small
$
 \Sigma_{\alpha_{i,k}} = \left[\begin{smallmatrix}
 U_r & U_n
 \end{smallmatrix}\right]  \left[\begin{smallmatrix}
\Lambda_r& 0\\
0 & 0
 \end{smallmatrix}\right]  \left[\begin{smallmatrix}
 U_r & U_n
 \end{smallmatrix}\right]^{T}
$\normalsize
~where \small$U_r \in \mathbb{R}^{3\times (3-n_{i,k})}$\normalsize~is comprised of the eigenvectors of $\Sigma_{\alpha_{i,k}}$ that are orthonormal. The columns of \small$U_n \in \mathbb{R}^{n_{i,k}}$\normalsize~are the complementary orthonormal basis that spans the null space of $\Sigma_{\alpha_{i,k}}$.
By substituting~\eqref{eq:handk} one can verify the following inequality:
\vspace{-1mm}
\begin{equation} \label{eq:upperbound1}
     \Delta_{i,k}\! \leq\!  \sqrt{(\mathbold{p}_i\!-\!\mathbold{h}_{i,k})^{T}\Tilde{\Sigma}_{\alpha_{i,k}}(\mathbold{p}_i\!-\!\mathbold{h}_{i,k})\!+\!k_{i,k}}\triangleq\Tilde{\Delta}_{i,k}
     \vspace{-1mm}
\end{equation}
where~\small$\Tilde{\Sigma}_{\alpha_{i,k}}$ is a positive definite matrix because
\vspace{-1mm}
\begin{equation*}
    \Tilde{\Sigma}_{\alpha_{i,k}} = \left[\begin{smallmatrix}
 U_r & U_n
 \end{smallmatrix}\right]  \left(\left[\begin{smallmatrix}
\Lambda_r& 0\\
0 & 0
 \end{smallmatrix}\right] + \left[\begin{smallmatrix}
0 & 0\\
0 & I_{n_{i,k}\times n_{i,k}}
 \end{smallmatrix}\right] \right)\left[\begin{smallmatrix}
 U_r & U_n
 \end{smallmatrix}\right]^{T}.
 \vspace{-1mm}
\end{equation*}
We further upper bound \eqref{eq:upperbound1} by adding a positive constant, $\mathbold{\iota}_{i,k} \triangleq\frac{2}{\sqrt{3}}\mathbold{1}^{T}|\Tilde{\Sigma}_{\alpha_{i,k}}^{1/2}\mathbold{p}_k - \mathbold{h}_k|$, to $\Tilde{\Delta}_{i,k}^2$ and obtain
\vspace{-1mm}
\begin{align*}
    \Tilde{\Delta}_{i,k}^{2} \leq  \Tilde{\Delta}_{i,k}^{2} + \frac{2}{\sqrt{3}}\mathbold{1}^{T}|\Tilde{\Sigma}_{\alpha_{i,k}}^{1/2}(\mathbold{p}_k - \mathbold{h}_k)|\leq \sqrt{\xi_{i,k}^{T}\xi_{i,k}}
    \vspace{-3mm}
\end{align*}
where \small$\xi_{i,k} \triangleq |\Tilde{\Sigma}_{\alpha_{i,k}}^{1/2}(\mathbold{p}_k-\mathbold{h}_{i,k})|  + \mathbold{1}\sqrt{\frac{k_{i,k}}{3}}\in \mathbb{R}^{3}$\normalsize. For the inequality to hold, the expression $\mathbold{\iota}_{i,k}$ must always be non-negative which is true by construction.
Further, let \small$\zeta_{i,k} = \mathbold{1}^{T}\xi_{i,k}\in \mathbb{R}$\normalsize, then  \small$\zeta_{i,k}^2 = (\xi_{i,k}^T \mathbold{1}) (\mathbold{1}^{T}\xi_{i,k}) = \xi_{i,k}^T\xi_{i,k} + 2\epsilon_{\xi}$\normalsize. If $\epsilon_{\xi} \geq0$, we can then state $\Delta_{i,k} \leq \zeta_{i,k}$ which completes the proof (since \small$\xi_{i,k} = [\xi_{i,k}^x,\xi_{i,k}^y, \xi_{i,k}^z]\in \mathbb{R}^{3}$\normalsize, then \small
$\epsilon_{\xi} = \xi_{i,k}^x\xi_{i,k}^y + \xi_{i,k}^x\xi_{i,k}^z + \xi_{i,k}^y\xi_{i,k}^z > 0\,\,
$\normalsize  
because \small$\xi_{i,k}^x,\xi_{i,k}^y, \xi_{i,k}^z \in \mathbb{R}_{+}$\normalsize).
\end{proof}

\end{document}